\documentclass[twoside]{article}

\usepackage[accepted]{aistats2020}

\usepackage[utf8]{inputenc} 
\usepackage[T1]{fontenc}    
\usepackage{hyperref}       
\usepackage{xcolor}
\hypersetup{
    colorlinks,
    linkcolor={red!50!black},
    citecolor={blue!50!black},
    urlcolor={blue!80!black}
}

\usepackage{url}            
\usepackage{booktabs}       
\usepackage{amsfonts}       
\usepackage{nicefrac}       
\usepackage{microtype}      
\usepackage{todonotes}

\usepackage{amsmath}

\def\eqref#1{equation~\ref{#1}}

\def\1{\bm{1}}

\def\vg{{\bm{g}}}

\def\vs{{\bm{s}}}

\def\vu{{\bm{u}}}
\def\vv{{\bm{v}}}

\def\mA{{\bm{A}}}
\def\mB{{\bm{B}}}

\def\mE{{\bm{E}}}

\def\mG{{\bm{G}}}

\def\mI{{\bm{I}}}

\def\mM{{\bm{M}}}

\def\mS{{\bm{S}}}

\def\mU{{\bm{U}}}
\def\mV{{\bm{V}}}
\def\mW{{\bm{W}}}

\DeclareMathAlphabet{\mathsfit}{\encodingdefault}{\sfdefault}{m}{sl}
\SetMathAlphabet{\mathsfit}{bold}{\encodingdefault}{\sfdefault}{bx}{n}

\def\gD{{\mathcal{D}}}

\def\gF{{\mathcal{F}}}

\def\gL{{\mathcal{L}}}

\def\gN{{\mathcal{N}}}

\def\gY{{\mathcal{Y}}}

\def\sP{{\mathbb{P}}}

\def\sZ{{\mathbb{Z}}}

\newcommand{\E}{\mathbb{E}}

\newcommand{\R}{\mathbb{R}}

\newcommand{\KL}{D_{\mathrm{KL}}}
\newcommand{\Var}{\mathrm{Var}}

\newcommand{\diag}{\mathop{\mathrm{diag}}\nolimits}
\newcommand{\vvv}{\textnormal{vec}}
\newcommand{\vepsilon}{\bm{\epsilon}}
\newcommand{\gMN}{\mathcal{MN}}
\newcommand{\Tr}{\textnormal{Tr}}
\newcommand{\ddd}{\mathrm{d}}
\newcommand{\subN}{\textnormal{sub}\gN}

\usepackage{bm}
\usepackage[normalem]{ulem}
\usepackage{todonotes}
\usepackage{enumitem}
\usepackage{amsthm}
\newtheorem{theorem}{Theorem}
\newtheorem{lemma}[theorem]{Lemma}
\newtheorem{definition}{Definition}

\newtheorem{property}{Property}
\usepackage{subfigure}
\usepackage{graphicx}

\usepackage[round]{natbib}

\newcommand\footnotenonumber[1]{%
  \begingroup
  \renewcommand\thefootnote{}\footnote{#1}%
  \addtocounter{footnote}{-1}%
  \endgroup
}

\begin{document}

\twocolumn[

\aistatstitle{Stochastic Neural Network with Kronecker Flow}

\aistatsauthor{ Chin-Wei Huang \And Ahmed Touati \And  Pascal Vincent} 

\aistatsaddress{ Mila \And  Mila \And Mila, Facebook Research}

\aistatsauthor{Gintare Karolina Dziugaite \And Alexandre Lacoste \And Aaron Courville }

\aistatsaddress{ Element AI \And  Element AI \And Mila, CIFAR Fellow}

]


\begin{abstract}
Recent advances in variational inference enable the modelling of highly structured joint distributions, but are limited in their capacity to  scale to the high-dimensional setting of stochastic neural networks. 
This limitation  motivates a need for scalable parameterizations of the noise generation process, in a manner that adequately captures the dependencies among the various parameters. 
In this work, we address this need and present the \emph{Kronecker Flow}, a generalization of the Kronecker product to invertible mappings designed for stochastic neural networks. 
We apply our method to variational Bayesian neural networks on predictive tasks, PAC-Bayes generalization bound estimation, and approximate Thompson sampling in contextual bandits. 
In all setups, our methods prove to be competitive with existing methods and better than the baselines.  
\end{abstract}

\section{Introduction}
\footnotenonumber{Work done while Chin-Wei was an intern at Element AI and Ahmed at Facebook Research. 
}
Stochastic neural networks (SNN) are a central tool in many subfields of machine learning, including
(1) Bayesian deep learning~\citep{mackay1992practical,Blundell2015,hernandez2015probabilistic,gal2016dropout},
(2) exploration in reinforcement learning~\citep{osband13posterior,osband16randomvalue,riquelme18showdown}, and 
(3) statistical learning theory such as PAC-Bayesian learning~\citep{mcallester1999pac,langford2001bounds,dziugaite2017computing}. 
Perturbations of the network parameters induce a distribution over the model, and this intrinsic uncertainty is the subject of great interest to machine learning practitioners and theoreticians alike.  
For example, deep Bayesian models are often used to adequately measure uncertainty, and determine whether the model itself is inherently familiar with the unseen data. 
This is especially important in the context of autonomous vehicles, where decisions must be made to meet specific safety standards~\citep{mcallister2017av}.
Conversely, the lack of confidence can be leveraged to efficiently guide exploration in reinforcement learning, via randomizing the approximate value function~\citep{azizzadenesheli18bdqn,touati2018randomized} or maximizing intrinsic rewards~\citep{houthooft16vime}.

Furthermore, a considerable proportion of statistical learning theory is devoted to understanding what implies generalization, or what constitutes an appropriate measure of complexity~\citep{telgarsky17,arora18,neyshabur27exploregen}. 
PAC-Bayesian learning theory~\citep{mcallester1999pac} specifically explores the generalization property of a randomized prediction rule, and has been recently studied in the context of stochastic neural networks~\citep{dziugaite2017computing}. 
In this particular study, the working hypothesis is that good generalization can be guaranteed on the premise that stochastic gradient descent~\citep{robbins1951} finds a solution that obtains certain structural properties, such as flatness.

For computational reasons, considerable effort has been devoted to modelling uncertainty through the injection of independent noise to the network parameters~\citep{graves11practicalvi,Blundell2015,kingma15local}. 
However, noise independence largely restricts the expressivity of the noise distribution and thus the resulting uncertainty measures are ill-calibrated~\citep{minka2005divergence,turner2011two}. 
Attempts have been made to correlate parameters of a neural network, including ~\citet{louizos2017multiplicative,krueger2017bayesian,pawlowski17implicit}, for example, by adapting expressive non-linear invertible transformations developed in the variational inference literature~\citep{rezende2015variational, kingma2016improved, huang2018neural}, or via implicit methods~\citep{goodfellow2014generative}. 
However, these methods are limited due to their inability to scale well.
\citet{louizos2017multiplicative}, for instance, resort to a specific multiplicative noise sampled from a lower dimensional space and have to use an auxiliary method to bound the entropy. 
\citet{krueger2017bayesian}, on the other hand, give up on injecting noise on the entire set of parameters and model the distribution of the scale and shift parameter of the pre-activations. 

In attempts to address some of the challenges articulated above and efficiently model the joint distribution of a network's parameters, 
we take inspiration from the Kronecker product, which we notice can be thought of as left-transforming a matrix via a linear map, and then right-transforming it using another linear map, thus providing us an efficient way to correlate the weight parameters. 
We propose the \emph{Kronecker Flow}, an invertible transformation-based method that generalizes the Kronecker product to its nonlinear counterparts.
Our contributions are as follows. 
\begin{enumerate}
    \item We extend the idea of Kronecker product to more general invertible mappings to induce non-linear dependencies, and apply this trick to parameterizing deep stochastic neural networks. 
    \item We apply our method to predictive tasks and show that our methods work better on larger architectures compared to existing methods.
    \item We are the first to apply flow-based methods to tighten the PAC-Bayes bound. We show that the KL divergence in the PAC-Bayes bound can be estimated with high probability, and demonstrate the generalization gap can be further reduced and explained by leveraging the structure in the parameter space.  
    \item Our methods prove to be competitive over other methods in approximate Thompson sampling in contextual bandit problems. 
\end{enumerate}

\section{Background}
Stochastic neural networks with parameter perturbation normally follow the stochastic process: $\Theta\sim q_\phi(\Theta)$, $y|x\sim p(y|x, \Theta)=f_\Theta(x)$, where $\Theta$ is the parameters of the neural network $f$, which outputs the prediction probability vector for classification or the predicted values for regression. 
We let $D=\{(x_i,y_i):{i\in[m]}\}$ be the training set of size $m$ \footnote{We use the notation $[n]$ to compactly describe the set of integers $\{1,2,\dots,n\}$.}, $H$ be the differential entropy $H[q]=- \E_q[\log q]$, $\beta>0$ be the coefficient controlling the amount of noise injected into the model and the degree of regularization, $l(y,\bar{y})$ be the loss function and $\hat{R}_{D}(\Theta)=\frac{1}{m}\sum_{i=1}^{m} l(y_i, f_\Theta(x_i))$ be the empirical risk.

\subsection{Variational Bayesian neural networks}
\label{sec:vi}
Variational Bayesian neural networks are a type of stochastic neural network. 
Bayesian inference updates our prior belief $p(\Theta)$ over the model parameters according to the Bayes rule $p(\Theta|D)\propto p(D|\Theta)p(\Theta)$, by incorporating information from the training set through the likelihood function $p(D|\Theta)$. 
Variational inference is a computational realization of Bayesian inference, which casts inference as 
an optimization problem, where one maximizes the variational lower bound (also known as the evidence lower bound, or the ELBO) on the log marginal likelihood:
\begin{align}
\label{eq:elbo}
\log p(D) \geq \E_{q_\phi}[ \log p(D|\Theta) + \log p(\Theta)] + H(q_\phi(\Theta)),
\end{align}
where $q_\phi$ is the variational approximate posterior and $p(D|\Theta)$ can be decomposed into $\prod_{i=1}^{m} p(y_i|x_i, \Theta)$ due to conditional independence assumption. 
The optimal $q$ is the true posterior, i.e. $q^*(\Theta)=\frac{p(D|\Theta)p(\Theta)}{p(D)}$.
In our case, we use $\Theta$ to parameterize a neural network. Prediction can be carried out via the (approximate) predictive posterior 
\begin{align*}
    p(y|x,D) & = \E_{\Theta\sim p(\Theta|D)} [p(y|x,\Theta)] \\
    & \approx\E_{\Theta\sim q_\phi(\Theta)}[p(y|x,\Theta)] \\
    & \approx \frac{1}{K}\sum_{k=1}^K p(y|x,\Theta_k)
\end{align*}
for $\{\Theta_k\}_{k\in[K]}$ drawn i.i.d. from $q_\phi(\Theta)$, where we use the variational distribution $q$ to approximate $p(\Theta|D)$ and a Monte Carlo estimate to estimate the integral. 

The prior distribution can be used to encode some form of inductive bias, such as one that is in favor of parameter values closer to some $\Theta_0$ chosen \emph{a priori}. 
We choose the prior to be an isotropic Gaussian, centered at the random initialization $\Theta_0$, i.e., $p(\Theta) = \gN(\Theta;\Theta_0, \lambda\mI)$. 
The entropy term ensures the variational posterior does not collapse to a point estimate. 
Both of them can be thought of as some form of regularizer, so we attach a coefficient $\beta$ in front of them as a hyperparameter \footnote{Like the $\lambda$ parameter in \citet{zhang2017noisy}}.

\subsection{PAC-Bayes generalization bound}
\label{sec:pac-bayes-bound}
Another use case of stochastic neural networks is to understand generalization, via  PAC-Bayes bounds. The aim is to bound a divergence  between the empirical risk, $\hat{\gL}[q]=\E_q[\hat{R}_D(\Theta)]$, and the risk measured on the \emph{true} distribution $\mathcal{D}$, $\gL[q]=\E_q[\E_\gD[l(y,f_\Theta(x))]]$. While this quantity is unbounded in the general case, assuming a bounded loss function $l$ (e.g.: zero-one loss), we can obtain a probabilistic bound that holds with probability $1-\delta$ over the choice of $D$, for $\delta >0$. More specifically, with probability $1-\delta$,  $\Delta(\hat{\gL}[q], \gL[q])\leq \Omega(\KL(q||p),m,\delta)$, where $\Omega$ is a measure of complexity that scales proportionally with the Kullback–Leibler (KL) divergence and $\Delta$ a measure of divergence (e.g.: square distance or convex functions~\citep{germain2009pac}).

For instance, \citet{dziugaite2017computing} minimize the following bound originally due to \citet{mcallester1999pac} and then tightened by \citet{langford2001bounds}:
\begin{theorem}
\label{thm:langford}
Let $l$ be the zero-one loss.
For any $\delta>0$ and data distribution $\gD$, and any distribution $p$ on the space of $\Theta$, with probability at least $1-\delta$ over the choice of a training set $D\sim\gD^m$, for all distributions $q$ on the space of $\Theta$,
\begin{align}
\label{eq:mc}
\KL(\hat{\gL}[q]||\gL[q])\leq\frac{\KL(q||p)+\log \frac{m}{\delta}}{m-1},
\end{align}
where the KL on the LHS is between two Bernoulli distributions, defined by the probability of performing an error. 
\end{theorem}
We refer to the above bound as the \emph{McAllester bound}. 
The KL divergence on the RHS of the bound, also known as the \emph{information gain}, tells us to what extent the posterior $q$ is dependent on the training data.
The sharper and more confident $q$ is, and the farther away it is from the prior $p$, the larger the KL will be, which in turn is reflected by the larger bound on the generalization gap.
This is consistent with traditional notion of bias-variance trade-off. 

Alternatively, we consider the following bound due to \citet{catoni2007pac}:

\begin{theorem}
With the $l$, $\delta$, $\mathcal{D}$, and $p$ as defined in Theorem~\ref{thm:langford}, and with a fixed $\beta>1/2$, the following bound holds with probability over $1-\delta$:
\begin{align}
\label{eq:catoni}
\gL[q] &\leq \frac{1}{1-\frac{1}{2\beta}} \left(\hat{\gL}[q] + \frac{\beta}{m} \left( \KL(q||p) + \ln \frac{1}{\delta} \right) \right).
\end{align}
\end{theorem}
We refer to this bound as the \emph{Catoni bound}. 
We notice the linear relationship (which is also noticed by \citet{germain2016pac}) between the empirical risk and the KL divergence. 
This allows us to make use of the linearity of expectation to perform change of variable (see the next section). 
We also note that the optimal $\beta$ in Equation~\ref{eq:catoni} is always larger than $1$, so the PAC-Bayes bound is actually more conservative than Bayesian inference in this sense.

\subsection{Normalizing flows}
\label{sec:nf}
Minimization of Equation \ref{eq:elbo} and \ref{eq:catoni} requires (\textit{i}) computing the gradient with respect to the parameter of the (PAC-)Bayesian posterior $\phi$, and (\textit{ii}) computing the entropy of $q$. 
One approach to do this is via \emph{change of variable} under an invertible mapping. 
Let $\vepsilon\sim q_0$ be a random variable in $\R^d$, and $\vg_\phi:\R^d\rightarrow\R^d$ be a bijection parameterized by $\phi$. 
Let $\Theta=\vg_\phi(\vepsilon)$ and $q_\phi$ be its density.
Then we can rewrite the loss function as \footnote{Since the weighting coefficient $\beta$ can be absorbed into the loss function $l$, we neglect it for simplicity now.}
\begin{align*}
\E_{\Theta}[\hat{R}_{D}(\Theta) + \log q_\phi(\Theta)]
& =\E_{\vepsilon}[\hat{R}_{D}(\vg_\phi(\vepsilon)) + \log q_0(\vepsilon) \\
& - \log\left|\det\frac{\partial \vg_\phi(\vepsilon)}{\partial\vepsilon}\right|],
\end{align*}
where we apply the change of variable (see Appendix~\ref{app:lotus} for the detailed derivation). 
The log-determinant (logdet) term ensures that we obtain a valid probability density function after $g_\phi$ is applied, which can be a sequence of invertible mappings itself, hence referred to as the \emph{normalizing flow}~\citep{rezende2015variational}.
This way, the random variable and the parameters are decoupled, so that we can differentiate the integrand to have an unbiased estimate of the gradient (fixing some $\vepsilon\sim q_0$).
We let $q_0$ be the standard normal.

\section{Kronecker Flows}
We consider maximizing the ELBO and minimizing the Catoni bound via normalizing flow-based SNNs. 
Conventionally, mean-field approximation using factorized distributions (such as multivariate Gaussian with diagonal covariance) has been well explored in the variational inference (VI) literature~\citep{Blundell2015}.
We are interested in better capturing the structure in the parameter space as restricted VI methods are known to exhibit overconfidence~\citep{minka2005divergence,turner2011two}.
However, the parameters of a neural network are usually very high dimensional (on the order of millions), requiring a novel way to parameterize the joint distribution over the parameters. 

In its general form, neural networks can be represented by a collection of tensors
i.e. $\Theta=\{\mW_l:{l\in[L]}\}$. 
While our method below can easily be generalized to high-dimensional tensors (such as for convolutional kernels), 
to simplify notation, we describe the matrix form.

\subsection{Linear Kronecker Flow}

The matrix-variate normal ($\gMN$) distribution generalizes  the multivariate normal distribution to matrix-valued random variables, such as weight matrices of a neural network~\citep{louizos2016structured}. 
Matrix normal is a multivariate normal distribution whose covariance matrix is a Kronecker product ($\otimes$), which allows us to model the correlation among the parameters to some degree. 

More concretely,  assume $\mE_{ij}\overset{\textnormal{i.i.d.}}{\sim} \gN(0,1)$ is an $n\times p$ random Gaussian matrix, and $\mA\in\R^{n\times n}$, $\mB\in\R^{p\times p}$ and $\mM\in\R^{n\times p}$ are real-valued matrices. 
Then $\mM + \mA\mE\mB$ has a matrix normal distribution, as 
$$\vvv(\mM + \mA\mE\mB)\sim \gN(\vvv(\mM), \mB^\top\mB\otimes\mA\mA^\top),$$
where $\vvv$ is the vectorization of a matrix that concatenates all the columns.  
This allows us to represent the covariance matrix in a more compact manner ($n^2p^2/2$ parameters versus $n^2/2+p^2/2$ parameters for Kronecker product).

\paragraph{Limitation of the Kronecker product.} The Kronecker product covariance matrix is not a strict generalization of diagonal covariance matrix. 
To observe this, let $\mU=\diag(\vu)$, $\mV=\diag(\vv)$ (this is the case of~\citet{louizos2016structured}), and $\mS=\diag(\vs)$, where $\vu\in\R^n_{>0}$, $\vv\in\R^p_{>0}$, and $\vs\in\R^{np}_{>0}$.
Then $\mU\otimes\mV$ is also a diagonal matrix of size $np\times np$.
Equating $\mU\otimes\mV=\mS$ to solve for $\vu$ and $\vv$ will result in $np$ nonlinear equations with $n+p$ variables, which can be over-determined for $n,p>2$.
For example, let $n=2,p=3$, and $\vs=[1,\epsilon,\epsilon,1,1,1]$ for some $\epsilon>0$. Then the nonlinear system below does not have a solution:
\begin{align*}
\mU\otimes\mV = \mS \quad \Longleftrightarrow & \quad
\vu_1\vv_1\overset{(a)}{=}1 \quad
 \vu_1\vv_2\overset{(b)}{=}\epsilon \quad 
\vu_1\vv_3\overset{(c)}{=}\epsilon \quad \\
& \vu_2\vv_1\overset{(d)}{=}1 \quad
\vu_2\vv_2\overset{(e)}{=}1 \quad 
\vu_2\vv_3\overset{(f)}{=}1
\end{align*}
To see this, dividing $(a)$ by $(b)$ and dividing $(d)$ by $(e)$ yield $\vv_1=\vv_2/\epsilon$ and $\vv_1=\vv_2$, respectively, which doesn't have a solution if $\epsilon\neq1$.
This is because the Kronecker product is essentially parameter sharing, which can heavily restrict the matrix it can represent. 

To remedy the above limitation, we can further decouple the reparameterization of the parameter matrix into two parts: (1) one that models the marginal variance and (2) one that models correlations. 
Assume $\mS\in\R^{n\times p}_{>0}$ is a positive-valued matrix, and let $\mW:=\mM + \mA(\mE\circ\mS)\mB$. 
Then $\vvv(\mW)$ is a Gaussian distribution with the following property, which is useful in calculating the KL divergence:
\begin{property}
\label{property:linear}
Let $\mW$ be given as above, with $\mu=\E[\vvv(\mW)]$ and $\Sigma=\Var(\vvv(\mW))$. Then
\begin{enumerate}[label=(P{{\arabic*}})]
    \item $\mu=\vvv(\mM)$, and \\ $\Sigma=(\mB^\top\otimes\mA) \diag(\textnormal{vec}(\mS^2)) (\mB\otimes\mA^\top)$
    \item $\det(\Sigma)=\det(\mA)^{2p}\det(\mB)^{2n}\prod_{ij}\mS_{ij}^2$
    \item $\Tr(\Sigma)=\sum_{ij}\left(\mA^2\mS^2\mB^2\right)_{ij}$
\end{enumerate}
\end{property}

See Appendix~\ref{app:linear} for the derivation and interpretation of the property. 
Naive implemetations of this can be inefficient and numerically unstable, as the entropy term involves computing the log-determinant of $\mA$ and $\mB$, requiring the standard automatic differentiation libraries to resort to singular value decomposition when the matrix is near-singular. 
Thus, we choose to parameterize $\mA$ and $\mB$ as lower triangular matrices\footnote{This is achieved by masking.} with ones on the diagonal, leaving the uncertainty to be modeled by $\mS$. 
This means $\det(\Sigma)=\prod_{ij}\mS_{ij}^2$.
\begin{figure*}[t!]
\centering
\subfigure[{Random Gaussian matrix}]{
\centering
\includegraphics[width=0.44\textwidth]{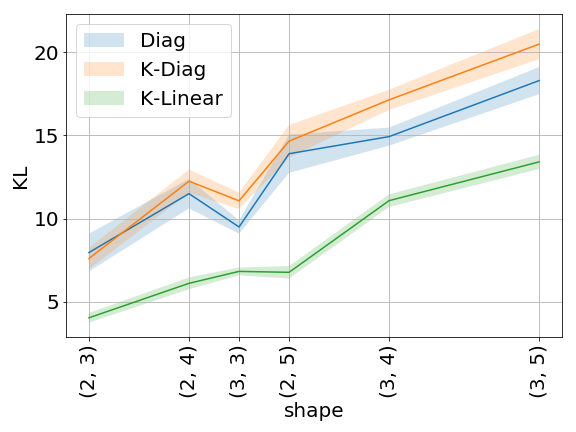}
\label{fig:linear_flows_2d}}
\hfill
\subfigure[{Random 3D Gaussian tensor}]{
\centering
\includegraphics[width=0.44\textwidth]{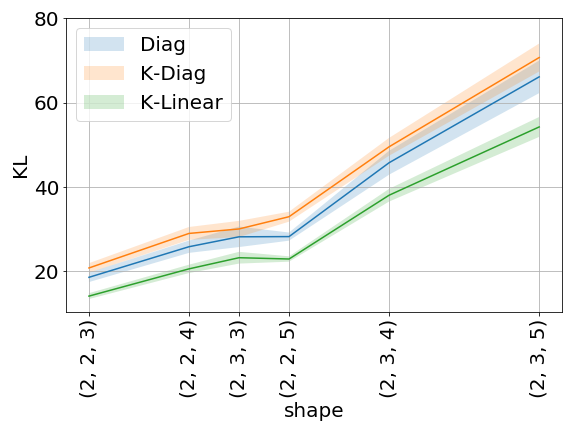}
\label{fig:linear_flows_3d}}
\caption{\small Minimizing KL divergence between $q$ and a randomly initialized distribution $p$. X-axis indicates the shape of the random matrix/tensor, sorted according to the dimensionality. The shaded area is the error bar with $0.1$-standard deviation away from the mean performance, averaged across 25 trials.}
\label{fig:linear_flows}
\end{figure*}
\paragraph{Simulation.} To validate the limited expressiveness of kronecker product, we randomly initialize a target density $p$ to be a multivariate Gaussian with mean zero, and covariance being the square of a random standard Gaussian matrix. 
We choose the dimensionality $d$ of the Gaussian such that it can be decomposed into a product of integers, and parameterize $q$ using independent Gaussian (dubbed Diag), the Kronecker product with diagonal $A$ and $B$ (K-Diag), and the Kronecker product with elementwise scaling (K-Linear). 
We minimize $\KL(q||p)$; see Figure~\ref{fig:linear_flows} for the results.
We also conduct the same experiment with 3D tensors (instead of matrices). 
We see that K-Diag consistently underperforms when compared to Diag, which indicates parameter sharing does restrict the family of distributions it can represent, and K-Linear is consistently better as it captures some correlation.

\subsection{Nonlinear Kronecker Flow}
In this section, we generalize the Kronecker product to more general non-linear mappings.
In Appendix~\ref{app:tri_map}, we make a connection to non-decreasing triangle maps~\citep{villani2008optimal} that are general enough to model any probability distributions. 

First, notice that left-multiplying $\mE$ by $\mA$ amounts to introducing linear correlation among the $n$ rows of $\mE$, applied to each of the $p$ columns. 
Likewise, right-multiplying $\mE$ by $\mB$ amounts to correlating column entries of each row of $\mE$. 
Inspired by this, we consider applying an invertible mapping to each row of the random weight matrix, and another invertible mapping to each column of the matrix.
We call this the \textbf{Kronecker Flow} \footnote{To differentiate this from K-Linear from the previous section, we refer to using non-linear $\vg$ as K-Nonlinear.}.

Specifically, let $\vg_A:\R^{n}\rightarrow\R^{n}$ and $\vg_B:\R^{p}\rightarrow\R^{p}$ be invertible mappings.
We define the matrix-matrix function $\mG:\R^{n\times p}\rightarrow\R^{n\times p}$ as $\mG_B(\mG_A(\mE^\top)^\top)$, with the following batch-operations (for $i\in[n]$ and $j\in[p]$):
$$\mG_A(\mE^\top)_{j:} := \vg_A(\mE_{:j}) \qquad \qquad \mG_B(\mE)_{i:} := \vg_B(\mE_{i:}) $$

It is easy to verify that $\mG$ is invertible.
Due to the partial dependency of $\mG_A$ and $\mG_B$, the Jacobians of the vectorized forms (after proper permutation) are block-diagonal, so we have
\begin{align*}
\det &\frac{\partial\textnormal{vec}(\mG(\mE))}{\partial\textnormal{vec}(\mE)}  \\&= \prod_{j\in[p]}\det\frac{\partial \vg_A(\mE_{:j})}{\partial \mE_{:j}}
\cdot
\prod_{i\in[n]}\det\frac{\partial \vg_B(\mG_A(\mE^\top)_{:i})}{\partial \mG_A(\mE^\top)_{:i}}.
\end{align*}
In practice, we use the volume preserving version of RealNVP~\citep{dinh2016density} and inverse autoregressive flow (IAF)~\citep{kingma2016improved} to parameterize $\vg_A$ and $\vg_B$ for our experiments \footnote{We also experimented with Block NAF by \citet{de2019block}, but did not include it in the manuscript: its performance was similar to IAF, but it is much slower to sample from (we adapted the open implementation from \citet{de2019block}).}. 
The K-Linear from the previous section can be thought of as using a linear map as $\vg_A$ and $\vg_B$.

\section{Concentration of empirical KL with normalizing flows}
In their study, \citet{dziugaite2017computing} use independent Gaussian for $q$ to minimize the McAllester bound, so they can compute the KL between Gaussians analytically. 
This is no longer feasible when we use more flexible families for $q$, such as normalizing flows. 
Moreover, a Monte Carlo estimate might result in underestimating the bound after inverting the KL between Bernoullis on the LHS of Equation~\ref{eq:mc} (which is a concave function; see Appendix A of \citet{reeb2018learning} for an illustration).
This necessitates a high probability bound on the concentration of the empirical estimate.  

In Section~\ref{sec:nf}, we have established $\KL(q||p)$ can be written in the following form
$$\resizebox{1.0\linewidth}{!}{%
$\E_{\vepsilon} \left[\log \gN(\vepsilon; \mathbf{0},\mI) - \log\left|\det\frac{\partial \vg_\phi(\vepsilon)}{\partial\vepsilon}\right|-\log \gN(\vg_\phi(\vepsilon); \mathbf{0}, \mI) \right],$}$$
where both $q_0$ and $p$ are standard Gaussian (the mean and variance can be absorbed into the invertible mapping $\vg$ if this is not the case). 

The first term in the KL can be computed analytically. The second term usually can be almost surely bounded (e.g. using Block neural autoregressive flows) so that we can use Hoeffding-type concentration or it can simply be made zero using e.g. volume preserving flows. The challenge now lies in the third term, which has a quadratic form $\frac{1}{2}||\vg(\vepsilon)||^2$, neglecting the normalizing constant.

Now assume $\vg$ is a $L_0$-Lipschitz~\footnote{The following flows are all Lipschitz (with Lipschitz activation functions): volume preserving version of \citet{dinh2016density,kingma2016improved}, \citet{berg2018sylvester}, \citet{behrmann2018invertible}, \citet{de2019block}, etc.}. 
Let $g(\vepsilon)=\frac{1}{\sqrt{2}}||\vg(\vepsilon)||$. 
Then $g$ is $L_0/\sqrt{2}$-Lipschitz:

\begin{align*}
\big|g(\vepsilon_1)-g(\vepsilon_2)\big|= & 
\frac{1}{\sqrt{2}}\big|||\vg(\vepsilon_1)|| - ||\vg(\vepsilon_2)||\big| \\
& \leq
\frac{1}{\sqrt{2}}||\vg(\vepsilon_1) - \vg(\vepsilon_2)||\leq
\frac{L_0}{\sqrt{2}}||\vepsilon_1-\vepsilon_2||.
\end{align*}

This is key in deriving a tail bound on $g^2$, as Lipschitz functions of canonical Gaussian random variables are \emph{sub-Gaussians}, meaning they have a tail that decays faster than a Gaussian random variable. 
The following theorem provides a concentration bound for the empirical average of $g^2$ similar to that of a Chi-square random variable, as $g^2$ (square of a sub-Gaussian) is sub-exponential. 
\begin{theorem}
\label{thm:kl_conc}
Let $g$ be defined as above with a Lipschitz constant $L=L_0/\sqrt{2}$. 
Let $\bar{g}^2=\frac{1}{K}\sum_{k=1}^K g_k^2$.
Then the following concentration bound holds
$$\sP(\bar{g}^2-\E[g^2]>\epsilon)\leq\exp\left(-\frac{K\epsilon^2}{2(4C^2+C\epsilon)}\right),$$
where $C=\left(6L^2+\frac{L}{\sqrt{\log2}}(\sqrt{d}+||\vg^{-1}(\mathbf{0})||)\right)^2$.
\end{theorem}
Note that in practice the empirical KL that we use is inversely scaled by the size of the training set $m$ (see Equation~\ref{eq:mc}), so the Lipschitz constant can be made small in practice to dominate the dimensionality.

\section{Experiments}
\label{sec:exp}
We evaluate our proposed method in the context of two prediction tasks (Section~\ref{sec:clf}), PAC-Bayes bound minimization (Section~\ref{sec:pac-bayes}) and contextual bandit (Section~\ref{sec:bandit}). 
For the two prediction tasks, we use the MNIST handwritten digit dataset~\citep{mnist} and CIFAR-10~\citep{cifar10}. 
See Appendix~\ref{app:exp_details} for a detailed description. 
\begin{table*}[]
    \begin{center}
    \small
    \caption{Test error with LeNet (\%) on MNIST and the first 5 classes of CIFAR-10. First 3 columns are from~\citet{louizos2017multiplicative}. K-Diag on CIFAR-5 diverged, so we did not include the result. }
    \label{tab:lenet}
    \begin{tabular}{c|ccc|cccc}
    \toprule
    {\bf Dataset} & {\bf L2} & {\bf FFG} & {\bf MNFG} & {\bf Diag} & {\bf K-Diag} & {\bf K-Linear} & {\bf K-Nonlinear} \\
    \midrule
    {\bf MNIST} & 0.6 & 0.9 & 0.7 & 0.92 & 0.67 & 0.70 & 0.60 \\
    {\bf CIFAR-5} & 24 & 22 & 16 & 19.0 & - & 16.8 & 17.4 \\
    \bottomrule
    \end{tabular}

    \end{center}
\end{table*}

\begin{table*}[]
    \centering
    \small
    \caption{Test error with modified version of VGG16 (\%) on CIFAR10. 
    First 4 columns are from \citet{zhang2017noisy}.
    \textbf{R} means regular training and \textbf{D} means training with data augmentation.}
    \label{tab:vgg16}
    \begin{tabular}{c|cccc|ccccc}
    \toprule
    {\bf Setup} & {\bf SGD} & {\bf KFAC} & {\bf BBB} & {\bf Noisy-KFAC}  & {\bf Diag} & {\bf K-Diag} & {\bf K-Linear} & {\bf K-Nonlinear} \\
    \midrule
    {\bf R} & 18.21 & 17.61 & 17.18 & 14.48 & 17.71 & 16.71 & 14.65 & 14.74 \\ 
    {\bf D} & 11.65 & 11.11 & 11.69 & 10.65 & 10.69 & 13.65 & 11.35 & 9.88  \\
    \bottomrule
    \end{tabular}
    
\end{table*}

\begin{table*}[]
	\centering
	\small
	\caption{PAC-Bayes bound estimation: We minimize the Pinsker bound (an upper bound on the McAllester bound) and the Catani bound using different flows, and estimate the McAllester bound at inference time using Newton's method.
	}
	\label{tab:pacbayes}
	\resizebox{1.0\textwidth}{!}{%
	\begin{tabular}{c|cccc|cccccc | cccc}
		\toprule
		\textbf{Bound} & \multicolumn{4}{c}{\bf Pinsker Bound}  & \multicolumn{6}{c}{\bf Catoni Bound} & \multicolumn{4}{c}{\bf Catoni Bound} \\
		\midrule
		\textbf{Flow}  & \multicolumn{2}{c}{\bf Diag} & \multicolumn{2}{c}{\bf K-Linear}  &\multicolumn{2}{c}{\bf Diag} &\multicolumn{2}{c}{\bf K-Linear} & \multicolumn{2}{c}{\bf K-Nonlinear} &{\bf D}&{\bf K-D}&{\bf K-L}&{\bf K-N}  \\
		\midrule
		$L-1$ & 1     & 2     & 1        & 2        & 1     & 2     & 1        & 2        & 1      & 2      & \multicolumn{4}{c}{LeNet-5}\\
		$\hat{\gL}[q]$         & 6.62  & 6.00  & 6.09     & 5.90     & 8.04  & 7.66  & 8.10     & 8.33     & 5.96   & 5.90   & 2.12 & 2.95 & 2.00 & 2.01 \\
		${\gL}[q]$        & 6.66  & 6.12  & 5.98     & 5.96     & 7.78  & 7.70  & 7.98     & 8.26     & 5.83   & 5.76   & 2.31 & 2.87 & 1.91 & 2.14\\
		bound       & 23.77 & 25.94 & 21.69    & 25.33    & 24.11 & 26.41 & 22.88    & 26.43    & 20.41  & 22.53  & 10.83 & 12.96 & 10.09 & 10.03\\
		KL          & 5968  & 7829  & 5292     & 7554     & 5001  & 6555  & 4334     & 5996     & 4725   & 5921  & 3177 & 3477 & 2913 & 2873\\
		\bottomrule
	\end{tabular}
	}
\end{table*}

\begin{table*}[] 
    \centering
    \small
    \caption{Cumulative regret incurred by different algorithms on the bandit benchmarks described in \citet{riquelme18showdown}. Values reported are the mean over 3 independent trials with standard error of the mean, normalized with respect to the performance of the uniform policy. }
    \label{tab:bandits}
    \resizebox{1.0\textwidth}{!}{%
    \begin{tabular}{c|cccccc}
    \toprule
    {\bf Bandit} & {\bf SGD} & {\bf fBNN} & {\bf Diag} & {\bf K-Diag} & {\bf K-Linear} & {\bf K-Nonlinear} \\
    \midrule
    {\bf Mushroom} & 4.06 $\pm$ 0.71 & 3.91 $\pm$ 0.89 & 2.16 $\pm$ 0.29 &
     2.41 $\pm$  0.73 & {\bf 1.85 $\pm$ 0.15} & 3.47 $\pm$ 0.47 \\
    {\bf Statlog} & 1.29 $\pm$ 0.20 & {\bf 0.73 $\pm$ 0.01} & 1.01 $\pm$ 0.01 &
    0.84 $\pm$ 0.06 & 0.81 $\pm$ 0.01 &  0.79 $\pm$ 0.04 \\
    {\bf Covertype} & 30.01 $\pm$ 0.21 & 32.03 $\pm$ 0.40 & 28.42 $\pm$ 0.30 & 29.19 $\pm$ 0.16 & 28.13 $\pm$ 0.12 & {\bf 28.06 $\pm$ 0.15}\\
    {\bf Financial} & 6.08 $\pm$ 0.47 & 7.27 $\pm$ 1.09 & 7.43 $\pm$ 0.57 &
    5.88 $\pm$ 0.25 & 5.88 $\pm$ 0.35 & {\bf 5.78 $\pm$ 0.28}\\
    {\bf Jester} & {\bf 56.24 $\pm$ 1.93} & 59.70 $\pm$ 2.48 & 59.34 $\pm$ 2.26 & 57.17 $\pm$ 1.81 & 57.66 $\pm$ 2.11 & 57.96 $\pm$ 2.58\\
    {\bf Adult} & 79.31 $\pm$ 0.47 & 84.45 $\pm$ 0.82 & 76.32 $\pm$ 0.09 &
    77.28 $\pm$ 0.01 & {\bf 75.94 $\pm$ 0.12} & 77.30 $\pm$ 0.26\\
    \bottomrule
    \end{tabular}
    }
\end{table*}

\subsection{Classification}
\label{sec:clf}
One benefit of Bayesian neural networks compared to the regular ones is that the trade-off between the prior and the likelihood is a form of regularization. 
In this section, we evaluate the generalization performance of our method applied to Bayesian neural networks. 
We consider two architectures: LeNet-5~\citep{mnist} and a modified version VGG-16~\citep{simonyan2014very} proposed by \citet{zhang2017noisy}. 

We first compare to the multiplicative normalizing flow (MNFG) proposed by \citet{louizos2017multiplicative}, applying our method to LeNet-5 (see Table~\ref{tab:lenet}). 
Our Diag matches the performance of their FFG (fully factorized Gaussian). 
K-Diag outperforms Diag in this case, perhaps due to the smaller number of parameters which makes it easier to optimize. 
K-Nonlinear yields the best generalization error in this case. 
On the CIFAR-5 experiment (we take the first 5 classes of CIFAR-10), our methods are on par with MNFG. 

Second, we compare with the noisy K-FAC proposed by~\citet{zhang2017noisy}, applying our methods to the larger architecture VGG-16 (see Table~\ref{tab:vgg16}). 
Noisy K-FAC applies an approximate natural gradient method. 
Despite this advantage, our methods (K-Linear and K-Nonlinear) have simiar prediction accuracy in the regular setup. 
We also include the results of data augmentation with horizontal flip and random crop where K-Nonlinear outperforms all the other methods.

\subsection{PAC Bayes bound minimization}
\label{sec:pac-bayes}

For the PAC-Bayes bound estimation, we minimize Equation~\ref{eq:catoni}.
We follow the recipe of~\citet{dziugaite2017computing}.
We upper bound the zero-one loss by cross-entropy divided by $\log|\gY|$ (where $|\gY|$ is the number of classes) to make the upper bound tight. 
We set the prior to be $\gN(\Theta_0,\lambda\mI)$, where $\Theta_0$ is the initial value of the parameters, and apply a union bound to tune the prior variance $\lambda$. 
We also tune the $\beta$ coefficient as a parameter during training \footnote{We are allowed to do so since we treat Equation~\ref{eq:catoni} as an optimization objective, rather than report it as a bound. We report the McAllester bound, which holds for any $q$, even if it depends on $\beta$. }, and report the McAllester bound for comparison (since it is the tightest). 
For more details, see \citet{dziugaite2017computing} for reference. 

We test with a multi-layer perceptron with 1 or 2 hidden layers with 600 neurons and LeNet-5, evaluated on the MNIST dataset (see Table~\ref{tab:pacbayes}). 
For further clarification, we follow the steps of \citet{dziugaite2017computing} by minimizing the McAllester bound, using Pinsker's inequality to bound the inverse of the Bernoulli KL (which we call the Pinsker bound). 
Since this bound has a square root in the complexity term, we can only use the Gaussian family with an analytic form of the KL. 
The result we have is slightly looser than \citet{dziugaite2017computing} since we have a 10-class problem and they deal with a binary version of MNIST. 
We see that the bound can indeed be improved by capturing the correlation among the parameters. 
We then compare to minimizing the Catoni bound, which is slightly looser since the linear relationship between the empirical risk and the KL term penalizes the latter more when the KL is larger. 
However, by modelling the non-linear dependencies, K-Nonlinear clearly outperforms the other methods (even compared to the ones minimizing the Pinsker bound). 
This indicates there exists a considerable amount of structure in the parameter space that may explain the gap between the test error and the generalization bound. 

We also notice that, despite the linear relationship, the Catoni bound focuses more on the complexity term than the ELBO. For example, the empirical risks of LeNet-5 in Table~\ref{tab:pacbayes} are much higher compared to the test loss of Table~\ref{tab:lenet}. 
The reasons are:
(1) the optimal $\beta$ in Equation~\ref{eq:catoni} is larger than 1 (depending on the relative value of the KL), and (2) to properly upper bound the zero-one loss, we scale down the cross-entropy loss by $\log|\gY|$ during optimization. 
This means a learning algorithm based on a tight PAC-Bayes risk bound cannot overfit by much; see \citet{dziugaite2017entropy} for a recent demonstration. 
A smaller value of $\beta$ could bring down the risk and empirical risk, resulting in a looser bound but usually better test performance. 
This trade-off between test set performance and the tightness of the bound is a general issue for generalization bounds.

One reason for the interest in PAC-Bayes bounds is that their optimization leads to training algorithms with generalization guarantees. 
The bounds, however, are considerably looser than held-out estimates. \footnote{A recent work by \citet{rivasplata2019pac} shows PAC-Bayes bound can potentially be made tight; however, we could not reproduce the results and the best bound using Gaussian prior was around 40\% in their setting.} 
Our work produces much tighter bounds by building flexible families of distributions on neural network weight matrices.

\subsection{Contextual bandit}
\label{sec:bandit}

Uncertainty modeling lies at the heart of the exploration-exploitation dilemma in sequential decision-making. 
In order to maximize its collected cumulative rewards, an agent should trade off exploring different actions and gaining more knowledge about the reward estimate vs. exploiting the current estimate and allocating resources to the actions that are likely rewarding. \emph{Thompson sampling} (TS) \citep{thompson1933likelihood}  is one the popular approaches that deals with the latter trade-off by maintaining posterior distribution over reward models and randomizing actions on the basis of their probability of being optimal. 

In this section, we investigate the effectiveness of our proposed method for performing an approximate Thompson sampling in the particular setting of contextual bandits. In the latter setting, at each time $t = 1 \ldots T$, the agent sees a $d$-dimensional context $X_t$, selects one of the $k$ available actions, $a_t$, and earns a reward $r_t$ generated by the environment. The agent aims to minimize its cumulative regret defined as $R = \mathbb{E}[\sum_{t=1}^T r^{\star}_t - r_t]$ where $r^\star_t$ is the highest expected reward given the context $X_t$ and the expectation is over the randomness of both environment and the agent's choice of actions. 

We compare different methods on a range of real-world bandit problems introduced by \citet{riquelme18showdown}. We train the models every 50 time steps for 200 iterations using a batch-size of 512. We ran each experiment with 3 different random seeds and we report the means and standard errors of cumulative regret normalized with respect to the uniform baseline in the table \ref{tab:bandits}. We include the \emph{functional variational Bayesian
neural networks} (fBNN), recently introduced by \citet{sun2019functional} as a baseline, and we use their open sourced implementation of {fBNN} in the bandit setting. 
From table \ref{tab:bandits}, we see that across the 6 bandit problems, our proposed method (K-Linear and K-Nonlinear) provides competitive and consistent results. They outperform other baselines in 4 problems out of 6.

\section{Conclusion}
In this work, we present the \emph{Kronecker Flow}, a flow-based method to induce complex distribution inspired by the Kronecker product. 
Our methods scale to larger architectures such as VGG-16 since it takes advantage of the shape of the parameters. 
We demonstrate our methods work better than vanilla Kronecker product with diagonal matrices on multiple setups, including classification and approximate Thompson sampling in contexual bandit, and prove to be competitive with existing methods in the Bayesian neural network literature. 
We are also the first to apply flow-based methods to obtain a tighter numerical generalization bound. 
Our work shows that the dependencies among network parameters constitute a non-negligible portion of the gap between risk and PAC-Bayes generalization bound. 

\section*{Acknowledgement}
CWH would like to thank Kris Sankaran for pointing to the TIS inequality for Gaussian concentration, which is a key component in deriving the tail bound on Lipschitz flows. 
\bibliography{main}
\bibliographystyle{apalike}

\clearpage
\appendix
\onecolumn

\section{Law of the unconscious statistician}
\label{app:lotus}

Let $(\Omega, \gF, \sP)$ be our probability space. 
Let $\vepsilon\in\R^d$ be a random variable following the (Lebesgue) density $q_0(\vepsilon)=\frac{\ddd \vepsilon_*\sP}{\ddd \mu}$ and $\vepsilon_*\sP$ being its pushforward measure, and write $\Theta=\vg_\phi(\vepsilon)\in\R^d$ 
with $q_\phi=\frac{\ddd(\vg_\phi\circ\vepsilon)_*\sP}{\ddd\mu}$ being its density and $(\vg_\phi\circ\vepsilon)_*\sP$ being its pushforward measure,  
and $A=\hat{R}_D(\Theta)+\log q_\phi(\Theta)\in\R$.
Then 
\begin{align*}
\E[A] &= \int_{\R^d} A \,\,\ddd(\vg_\phi\circ\vepsilon)_*\sP 
&&= \int_{\R^d} \left(\hat{R}_D(\Theta) + \log q_\phi(\Theta) \right) q_\phi(\Theta) \,\ddd\Theta 
\\
&= \int_{\R^d} A\circ\Theta \,\,\ddd\vepsilon_*\sP
&&= \int_{\R^d} \left(\hat{R}_D(\vg_\phi(\vepsilon)) + \log q_\phi(\vg_\phi(\vepsilon)) \right) q_0(\vepsilon) \,\ddd\vepsilon \\
&{}&&=\int_{\R^d} \left(\hat{R}_{D}(\vg_\phi(\vepsilon)) + \log q_0(\vepsilon) - \log\left|\det\frac{\partial \vg_\phi(\vepsilon)}{\partial\vepsilon}\right|\right)q_0(\vepsilon)\,\ddd\vepsilon
\end{align*}

\section{Derivation and interpretation of Property~\ref{property:linear}}
\label{app:linear}
We first derive Property~\ref{property:linear} algebraically, and give an interpretation that can be genralized to higher dimensional tensor operation. 
Recall that we have the following givens:
\begin{itemize}
    \item Assume $\mE_{ij}\overset{\textnormal{i.i.d.}}{\sim} \gN(0,1)$ is a $n\times p$ random Gaussian matrix. 
    \item Assume $\mA\in\R^{n\times n}$ and $\mB\in\R^{p\times p}$. 
    \item $\mS\in\R^{n\times p}_{>0}$.
    \item $\mM\in\R^{n\times p}$.
\end{itemize}

If we rescale $\mE$ elementwise by $\mS$ before inducing the column-wise and row-wise correlation, we have: (superscript is Hadamard power)
\begin{align*}
    \Sigma:=\Var(\textnormal{vec}(\mM + \mA(\mE\circ\mS)\mB)) 
    &=\Var( (\mB^\top\otimes\mA)\textnormal{vec}(\mE\circ\mS))\\
    &=(\mB^\top\otimes\mA) \diag(\textnormal{vec}(\mS^2)) (\mB^\top\otimes\mA)^\top \\
    &=(\mB^\top\otimes\mA) \diag(\textnormal{vec}(\mS^2)) (\mB\otimes\mA^\top)
\end{align*}

If $\mS$ is a matrix of ones, the RHS equals $(\mB^\top\otimes\mA) (\mB\otimes\mA^\top)= (\mB^\top\mB)\otimes(\mA\mA^\top)$, which is the covariance of the matrix normal. 

Generally, $\mS$ might not be a matrix of ones.
But we can still compute the determinant and trace of the covariance matrix (useful in computing KL):

$$\det(\Sigma)=\det(\mA)^{2p}\det(\mB)^{2n}\prod_{ij}\mS_{ij}^2$$
\begin{align*}
    \textnormal{Tr}(\Sigma)
    &=\sum_i \Sigma_{ii} \\
    &=\sum_i \sum_{j} (\mB^\top\otimes\mA)_{ij}\textnormal{vec}(\mS^2)_j (\mB\otimes\mA^\top)_{ji}\\
    &=\sum_i \sum_{j} (\mB^{2\top}\otimes\mA^2)_{ij}\textnormal{vec}(\mS^2)_j\\
    &=\sum_i\left( (\mB^{2\top}\otimes\mA^2)\textnormal{vec}(\mS^2) \right)_i\\
    &=\sum_{ij}\left(\mA^2\mS^2\mB^2\right)_{ij}
\end{align*}

\paragraph{Interpretation of the determinant and trace.}
The determinant measures the change in volume due to the linear map. 
Since each operation (elementwise multiplication with $\mS$, left-multiplication with $\mA$, and right-multiplication with $\mB$) is an invertible map, the determinant of the composition is a product of determinants. 
After elementwise multiplication with $\mS$ (hence $\prod \mS_{ij}$)\footnote{The power 2 comes from the fact that we are looking at the determinant of the covariance. Direct computation of the likelihood involves $\frac{1}{2}\log\det(\Sigma)$, which is equivalent to the log-determinant of the invertible map.}, we apply the same linear map $\mA$ to the columns of $(\mE\circ\mS)$; in vector form, this corresponds to left-multiplication with a block diagonal of $p$ $\mA$'s, hence $\det(\mA)^{p}$. 
The same reasoning explains $\det(\mB)^n$.

The trace of the covariance can be written as $\Tr(\Sigma)=\sum_{ij}\Var(\mW_{ij})$, i.e. the sum of marginal variances. 
Each of the $\mW_{ij}$ is a linear combination of the entries of $\mE$, which have unit variance and are uncorrelated, so by the additive property of variance of sum of uncorrelated random variables and the quadratic scaling property of variance, $\Var(\mW_{ij})=A_{i:}^2\mS^2\mB_{:j}^2$.

\section{Connection to triangular maps.} 
\label{app:tri_map}
Much of the recent work on normalizing flows has been dedicated to inverse autoregressive transformations~\citep{kingma2016improved,huang2018neural,muller2018neural,de2019block,jaini2019sum}, as they are general enough to induce any density function~\citep{hyvarinen1999nonlinear,Bogachev_2005,villani2008optimal}. 
When such transformations are used for $\vg_A$ and $\vg_B$, the overall transformation $\mG$ is also a triangle map, since $\mG(\mE)_{ij}$ depends on $\mE_{i'j'}$ for $i'\leq i$ and $j'\leq j$. 
Such a function has some ``blind spots'' similar to the ones discovered by~\citet{van2016conditional}. 
One avenue for improvement is to design a transformation that increase the connectivity. 
Another avenue for improvement is to condition each (row-wise or column-wise) transformation on a learnable embedding of the row/column, such that each row/column is transformed by a slightly different function than another \footnote{We try this idea in the preliminary stage of the project, but find it harder to optimize. This is potentially due to the extra parameters that have to be learned.}.

\section{Tail bound of empirical KL}
\label{app:bound_kl}

We begin with some preliminaries and lemmas in Section~\ref{app:basics}, and prove the main result in Section~\ref{app:proof_kl}.

\subsection{Basic tail bounds and Bernstein inequality}
\label{app:basics}
The tools developed in this section is to translate the coefficients (such as variance) of sub-Gaussian random variables and sub-exponential random variables. 
We start with the definition of sub-Gaussians:
\begin{definition}
We write $X\sim\subN(L^2)$ if $X$ is a random variable satisfying 
$$\sP(|X|>t)\leq 2\exp\left(-\frac{t^2}{2L^2}\right)$$
\end{definition}

We write $\Gamma(\cdot)$ as the Gamma function: $\Gamma(z)=\int_0^\infty e^{-u}u^{z-1}\ddd u$.
Note that for positive integers $z$, $\Gamma(z)=(z-1)!$. 
The following lemma gives an upper-bound on the moments of a sub-Gaussian. 
\begin{lemma}
\label{lem:subgauss_moment}
For $X\sim\subN(L^2)$, for any integer $p\geq1$, $\E[|X|^p]\leq(2L^2)^{p/2}p\Gamma(p/2)$.
\end{lemma}
\begin{proof}
Since $|X|^p$ is non-negative, similar to Lemma~\ref{lem:subexp_moment}, we have
\begin{align*}
\E[|X|^p]&= \int_{0}^\infty\sP(|X|^p\geq s)\ddd s=\int_0^\infty\sP(|X|\geq t)pt^{p-1} \ddd t \\
&\leq 2p\int_0^\infty e^{-{t^2}/{2L^2}}t^{p-1} \ddd t 
=\leq p(2L^2)^{p/2}\int_0^\infty e^{-u}u^{p/2-1}\ddd u = p(2L^2)^{p/2}\Gamma(p/2)
\end{align*}
where we let $s=t^p$ and $u=t^2/2L^2$. 
\end{proof}

The following definition is the main tool for translating the coefficients. 
\begin{definition}
Let $X$ be a random variable.
For integer $k\geq1$, define the $\psi_k$-Orlicz norm as
$$||X||_{\psi_k}:=\inf\{t>0:\E[\exp(|X|^k/t^k)]\leq2\}$$
i.e, the smallest constant $t>0$ for which the super-exponential moment of $X^k/t^k$ is bounded by 2.
The Orlicz norm is infinity if there's no finite $t$ for which $\E[\exp(|X|^k/t^k)]$ exists. 
\end{definition}
It is easy to verify that the Orlicz norm is indeed a norm. 
We call $||\cdot||_{\psi_2}$ the sub-Gaussian norm, and $||\cdot||_{\psi_1}$ the sub-exponential norm. 
Note that $||X^2||_{\psi_1} = ||X||_{\psi_2}^2$.

The following lemma upper bounds the sub-Gaussian norm by its variance. 
\begin{lemma}
\label{lem:subgauss_norm}
If $X\sim\subN(L^2)$, $||X||_{\psi_2}\leq 6L^2$. 
\end{lemma}
\begin{proof}
By power series expansion of the exponential function,
\begin{align*}
\E[e^{cX^2}]=1+\sum_{p=1}^\infty\frac{c^p\E[X^{2p}]}{p!}
\leq1+\sum_{p=1}^\infty\frac{c^p}{p!} 2(2L^2)^pp!
=1+2\sum_{p=1}^\infty(2cL^2)^p
\end{align*}
where we used Lemma~\ref{lem:subgauss_moment} for the inequality. 
The RHS converges and is equal to $2$ if $c=1/6L^2$. 
Thus, $||X||_{\psi_2}\leq 6L^2$.
\end{proof}

The following lemma gives an upper bound on the moments of sub-exponential random variables. 
\begin{lemma}
\label{lem:subexp_moment}
If for some $C>0$, $\E[\exp(|X|/C)]\leq2$, then $\E[|X|^p]\leq 2C^pp!$.
\end{lemma}
\begin{proof}
By Markov's inequality, $$\sP(|X|>t)\leq \frac{\E[\exp(|X|/C)]}{\exp(t/C)}\leq 2e^{-t/C}$$
For $p\in\sZ^+$, since $|X|^p$ is non-negative, 
\begin{align*}
\E[|X|^p] &= \int_{0}^\infty\sP(|X|^p\geq s)ds
= \int_{0}^\infty\sP(|X|\geq t)pt^{p-1}dt \\
&\leq 2p\int_{0}^\infty e^{-t/C}t^{p-1}dt =2pC^p\int_{0}^\infty e^{-u}u^{p-1} du
=2pC^p\Gamma(p)=2C^pp!
\end{align*}
where we let $s=t^p$ and $u=t/C$.
\end{proof}

Finally, we derive a concentration bound for sub-exponential random variables. 
\begin{theorem}
\label{thm:subexp_conc}
(Bernstein's inequality for sub-exponential random variables) Let $(X_i)_{i\in[n]}$ be independent real-valued random variables satisfying $\E[\exp(|X|/C)]\leq2$ for some $C>0$, with mean $\mu_X=\E[X]$, and let $\bar{X}=\frac{1}{n}\sum_{i=1}^nX_i$. 
Then, for any $\epsilon>0$, the following concentration bound holds:
$$\sP(\bar{X}-\mu_X > \epsilon) \leq \exp\left(-\frac{n\epsilon^2}{2(4C^2+C\epsilon )}\right)$$
\end{theorem}
\begin{proof}
Let $\nu=4nC^2$ and $c=C$. 
Then by Lemma~\ref{lem:subexp_moment}, $\sum_{i=1}^n\E[X_i^2]\leq n\cdot4C^2=\nu$ and for integers $p>2$:
$\sum_{i=1}^n \E[|X_i|^p] \leq 2nC^pp! = \nu C^{p-2}p!/2 = \nu c^{p-2}p!/2$. 
Then by Corollary 2.11 of~\cite{boucheron2013concentration}, we have
$$\sP(\bar{X}-\mu_X > \epsilon)=\sP\left(\sum_{i=1}^n(X_i-\mu_X)>n\epsilon\right)\leq \exp\left(-\frac{n\epsilon^2}{2(4C^2+C\epsilon )}\right)$$
\end{proof}

\subsection{Proof of Theorem~\ref{thm:kl_conc}}
\label{app:proof_kl}
Since $\bar{g}(\vepsilon):=g(\vepsilon)-\E[g(\vepsilon)]$ is $L$-Lipschitz, according to Theorem 5.5 and 5.6 of \citet{boucheron2013concentration}, $\bar{g}\sim\subN(L^2)$.
And we have that
$$||g^2||_{\psi_1}=||g||_{\psi_2}^2=||\bar{g}+\E[g]||_{\psi_2}^2\leq\left(||\bar{g}||_{\psi_2} + \frac{\E[g]}{\sqrt{\log 2}}\right)^2$$
due to triangle inequality of the norm.
Now since $g$ is $L$-Lipschitz, its expectation can be bounded by
$$\E[g]=\E\left[\frac{1}{\sqrt{2}}||\vg(\vepsilon)-\mathbf{0}||\right]\leq L\E\left[||\vepsilon-\vg^{-1}(\mathbf{0})||\right]\leq L(\E\left[||\vepsilon||\right]+||\vg^{-1}(\mathbf{0})||) $$
Since $\vepsilon$ is standard-normally distributed, $||\vepsilon||$ follows the chi distribution with $d$ degrees of freedom, which has an expectation that can be upper-bounded using Gautschi's inequality (using Wendel's version of the upper bound):
$$\E[||\vepsilon||]=\sqrt{2}\frac{\Gamma((d+1)/2)}{\Gamma(d/2)}\leq \sqrt{2}\left(\frac{d}{2}\right)^{1/2}=\sqrt{d}$$
Combining the above and using Lemma~\ref{lem:subgauss_norm}, we have
$$||g^2||_{\psi_1}\leq\left(6L^2+\frac{L}{\sqrt{\log2}}(\sqrt{d}+||\vg^{-1}(\mathbf{0})||)\right)^2$$
Setting $C$ to be the RHS and applying Theorem~\ref{thm:subexp_conc} yield the desired result.

\section{Experimental Details}
\label{app:exp_details}
For the predictive tasks (Section~\ref{sec:clf}), we use a cosine annealing schedule for the learning rate, scaling down to 0.01 of the initial learning rate, and pretrain a deterministic network for 10 epochs using the Adam optimizer with a learning rate of 0.001, to initialize the mean of the Gaussian $q_0$, and train $q$ for 200 epochs.

\paragraph{LeNet-5 MNIST.} 
We use a linear annealing schedule of the $\beta$ coefficient (from 0 back to 1) for 50,000 iterations. 
We use the Adam optimizer with a learning rate of 0.0005. 
The result we get for K-Linear uses polyak averaging with exponential decay coefficient 0.995. 
We use the volume preserving version of RealNVP for the K-Nonlinear. 
We use the standard Gaussian prior for $p$.
\paragraph{LeNet-5 CIFAR-5}
We use the same architecture as \citet{louizos2017multiplicative} (192 convolutional kernels and 1,000 hidden units for the fully connected layers). 
We use a linear annealing schedule of the $\beta$ coefficient (from 0 back to 1) for 20,000 iterations for Diag, and no annealing for K-Linear and K-Nonlinear. 
We use the Adam optimizer with a learning rate of 0.0003, 0.0003, 0.0005 for Diag, K-Linear and K-Nonlinear, respectively. 
We use the volume preserving version of RealNVP for K-Nonlinear.
We use the standard Gaussian prior for $p$.
\paragraph{VGG-16 CIFAR-10}
We use the modified version of VGG-16 proposed by \citet{zhang2017noisy}.
We use a learning rate of 0.0005 for all experiments but K-Nonlinear in the regular setup (where we use 0.001).
We use the isotropic Gaussian prior with variance being 0.1, and set $\beta$ to be [0.5, 0.1, 0.1, 0.5] in the regular setup and [0.5, 0.1, 0.1, 0.1] in the data augmented setup for Diag, K-Diag, K-Linear, and K-Nonlinear, respectively. 
We use the volume preserving version of RealNVP for the K-Nonlinear. 
\paragraph{PAC-Bayes MLP}
We follow the same steps as \citet{dziugaite2017computing}, except we did not discretize the prior variance after tuning. 
In practice this does not affect the bound much. 
We also did not initialize the mean of $q_0$ in our setup using SGD for our experiments.
We train the stochastic network for 300 epochs, with a learning rate of 0.002. 
The bound holds with probability at least 0.965 over the choice of prior and the training set.
The $b$ and $c$ coefficients in \citet{dziugaite2017computing} are set as 100 and 0.1. 
We use the volume preserving version of IAF for the K-Nonlinear.
\paragraph{PAC-Bayes LeNet-5}
The same setup as PAC-Bayes MLP, except with polyak averaging with coefficient 0.995. 
We use the volume preserving version of IAF for the K-Nonlinear.

\paragraph{Bandit Benchmark}

All the models share the same architechture: one hidden layer with 50 units. We use the volume preserving version of RealNVP for K-NonLinear. We train models every 50 time steps for 200 training iterations using a batch-size of 512.

\begin{table}[h] \small
 \caption{
 \label{table: bandit dataset} Description of bandit problem: number of actions and number of contexts used for experiments. Comparing to \citet{riquelme18showdown} benchmark, we restrict ourself to 50000 contexts for Covertype instead of 150000 contexts. 
 }
\begin{center}
\begin{tabular}{ccc}
 \hline
 Bandit problem & number of actions & number of contexts \\
 Mushroom & 2 & 50000 \\
 Statlog & 7 & 43500 \\
 Covertype & 7 & 50000 \\
 Financial & 8 & 3713 \\
 Jester & 8 & 19181 \\
 Adult & 14 & 45222 \\
 \hline
 \end{tabular}
 \end{center}
 \end{table}
 
 \begin{table} [h]
    \centering
    \small
    \caption{Additional results: Cumulative regret incurred by different algorithms on the bandit benchmarks described in \citet{riquelme18showdown}. Values reported are the mean over 5 independent trials with standard error of the mean, normalized with respect to the performance of the uniform policy. We use the same hyperparameters for different algorithms without any finetuning: learning rate = 0.0001 and 100 training epochs.}
    \label{tab:bandits_without_tuning}
    \resizebox{1.0\textwidth}{!}{%
    \begin{tabular}{c|ccccc}
    \toprule
    {\bf Bandit} & {\bf SGD}  & {\bf Diag} & {\bf K-Diag} & {\bf K-Linear} & {\bf K-Nonlinear} \\
    \midrule
    {\bf Mushroom} & {\bf 1.82 $\pm$ 0.53} & 2.12 $\pm$ 0.13  &
     2.12 $\pm$ 0.47 & 2.20 $\pm$ 0.18 & 3.55 $\pm$ 0.74 \\
    {\bf Statlog} & 3.31 $\pm$ 1.27 & 4.12 $\pm$ 0.16 &
    {\bf 1.23 $\pm$ 0.06} & 3.49 $\pm$ 0.16 &  1.33 $\pm$ 0.02 \\
    {\bf Covertype} & 31.70 $\pm$ 0.17 & 34.64 $\pm$ 0.16 & 30.82 $\pm$ 0.21  & 32.84 $\pm$ 0.16 & {\bf 29.24 $\pm$ 0.10}\\
    {\bf Financial} & 20.31 $\pm$ 2.24  & 27.60 $\pm$ 1.37 &
    {\bf 11.83 $\pm$ 0.67}  & 25.10 $\pm$ 1.10  & 13.10 $\pm$ 0.36\\
    {\bf Jester} & 56.90 $\pm$ 1.20 & 59.26 $\pm$ 1.38 & 57.22 $\pm$ 1.35 &
    58.16 $\pm$ 1.20 & {\bf 56.88 $\pm$ 1.91}\\
    {\bf Adult} & 78.70 $\pm$ 0.46 & 79.16 $\pm$ 0.19 &
    77.30 $\pm$ 0.18  & {\bf 76.98 $\pm$ 0.04} & 77.83 $\pm$ 0.19\\
    \bottomrule
    \end{tabular}
    }
\end{table}

\newpage

\begin{figure}
    \centering
    \includegraphics[width=0.48\linewidth]{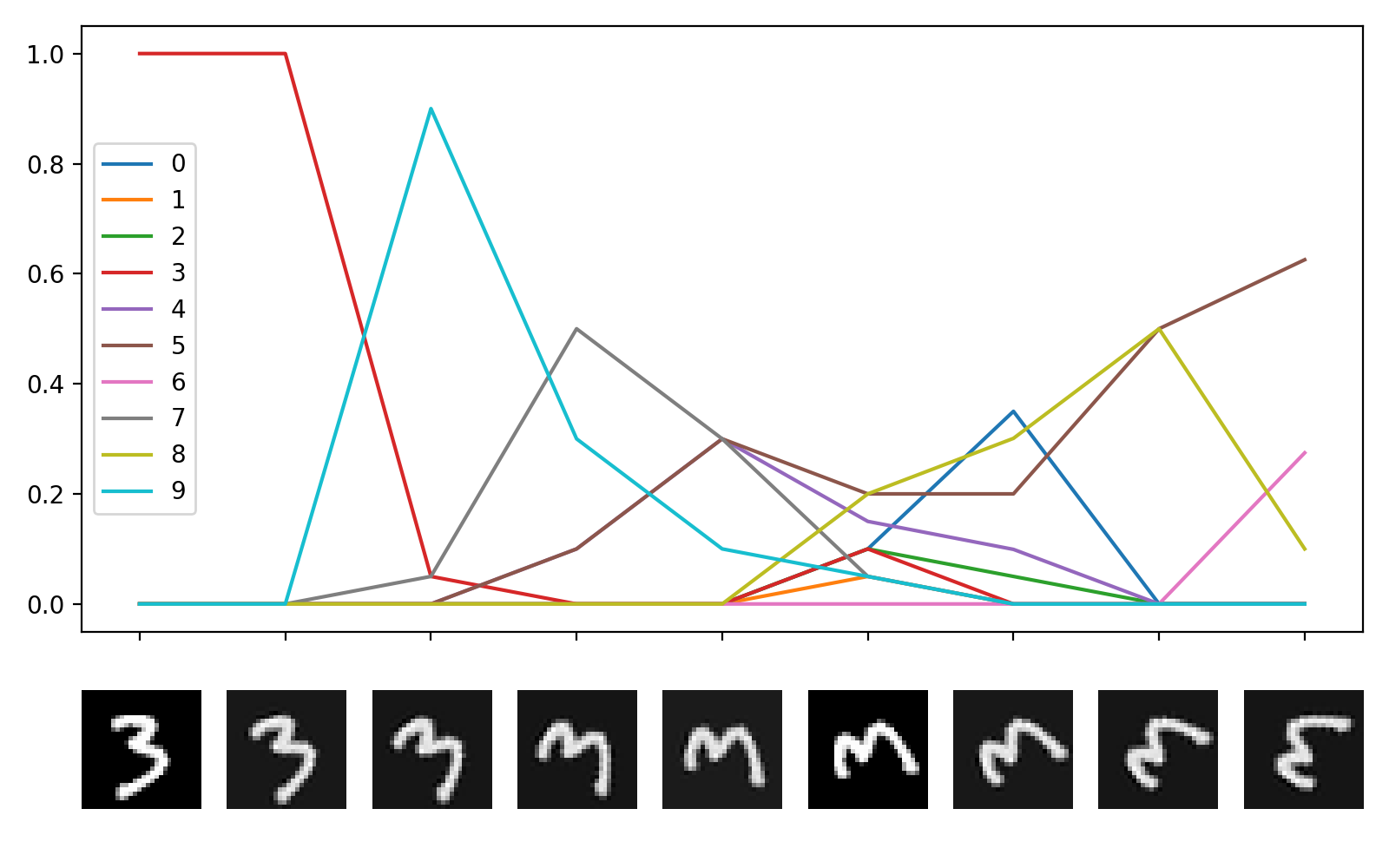}
    \hfill
    \includegraphics[width=0.48\linewidth]{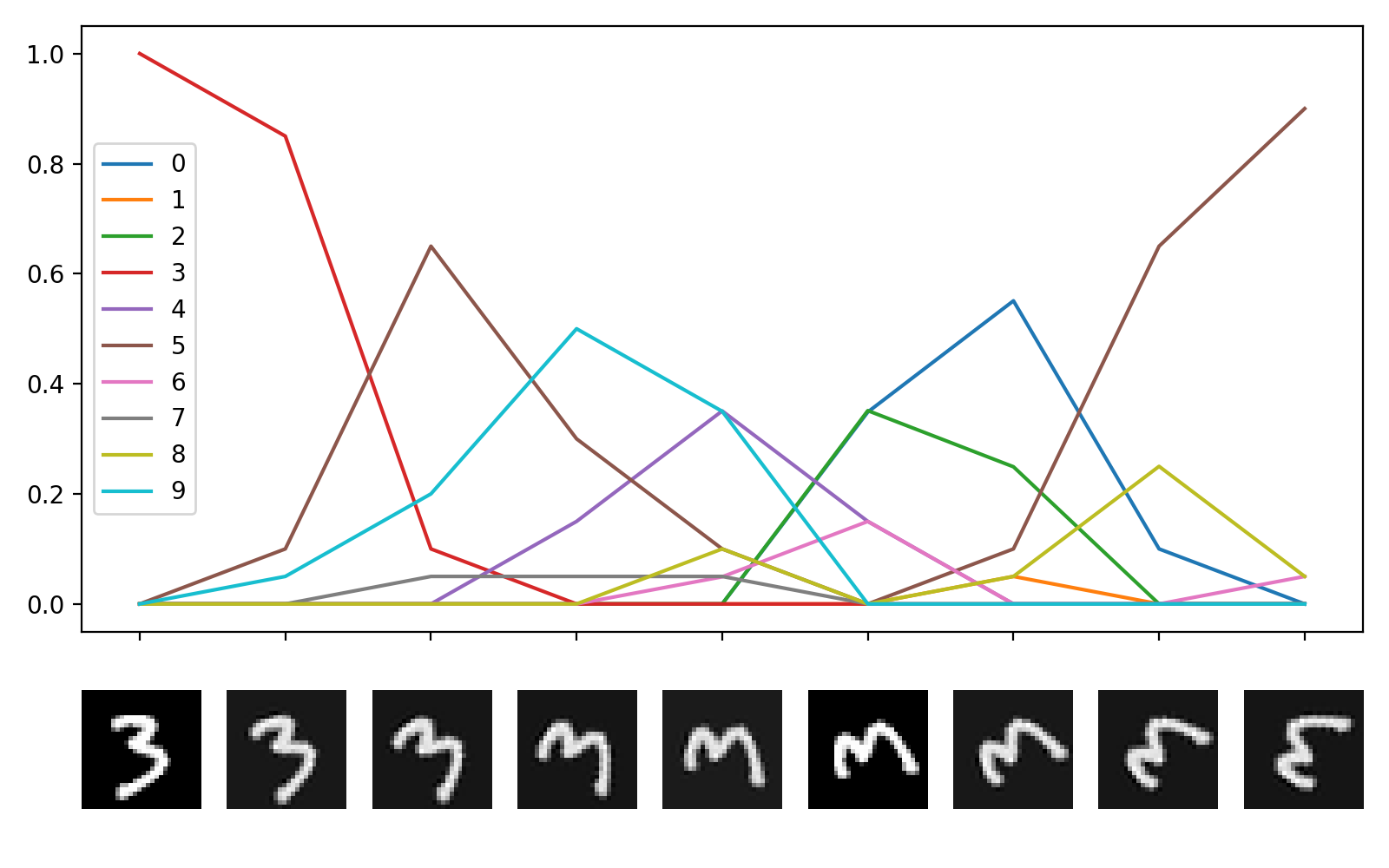}
    \includegraphics[width=0.48\linewidth]{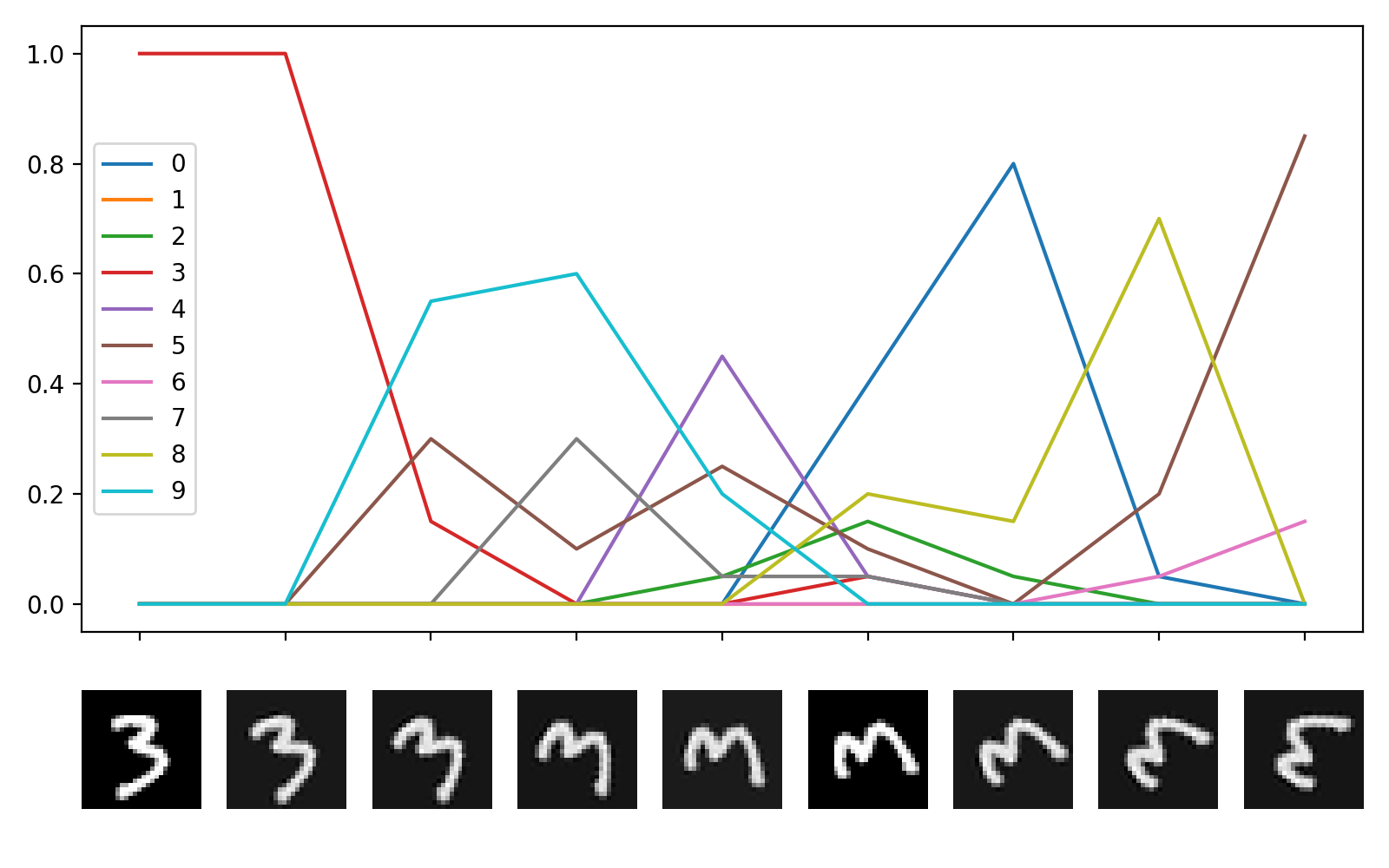}
    \hfill
    \includegraphics[width=0.48\linewidth]{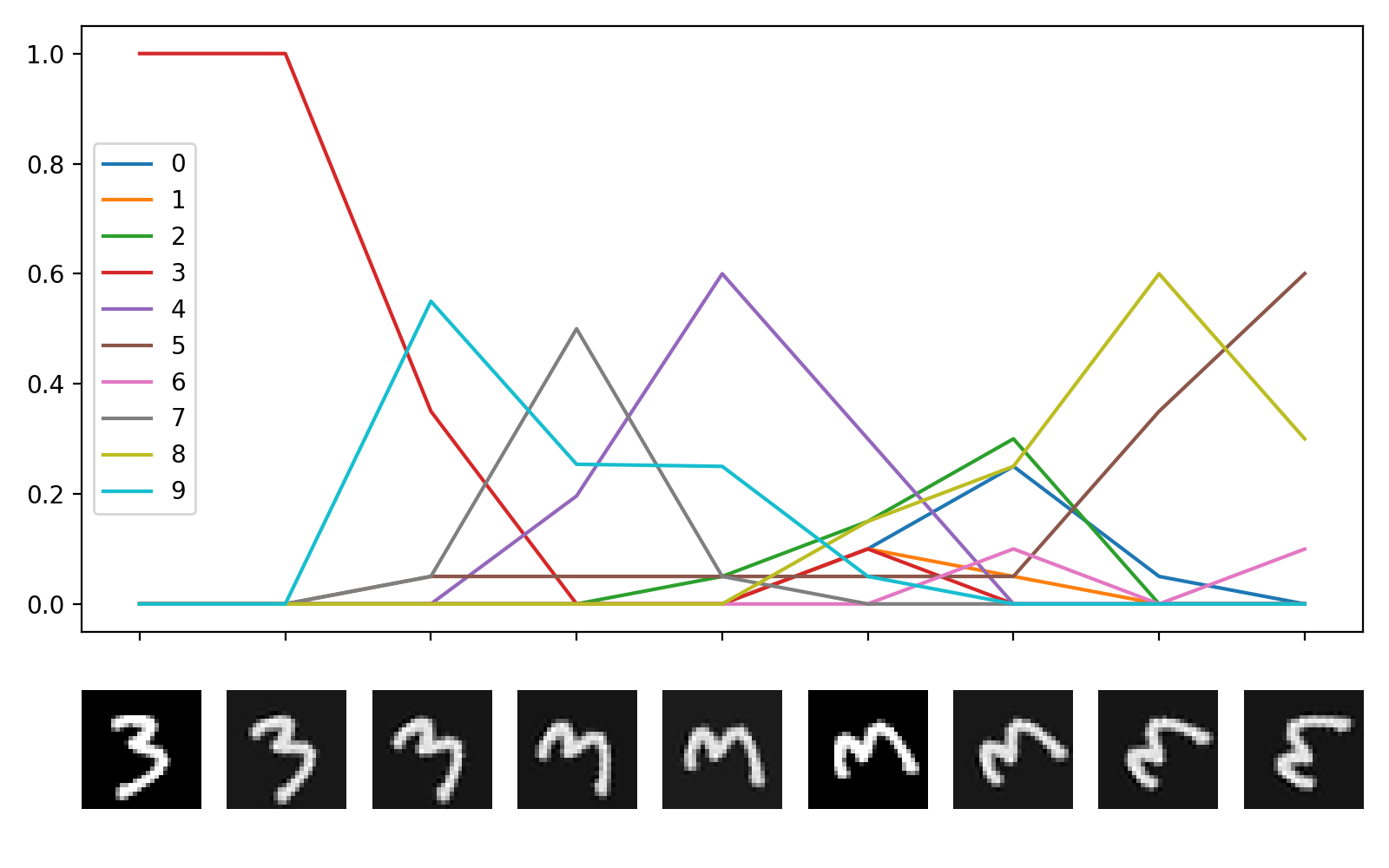}
    
    \vspace{2cm}
    
    \includegraphics[width=0.48\linewidth]{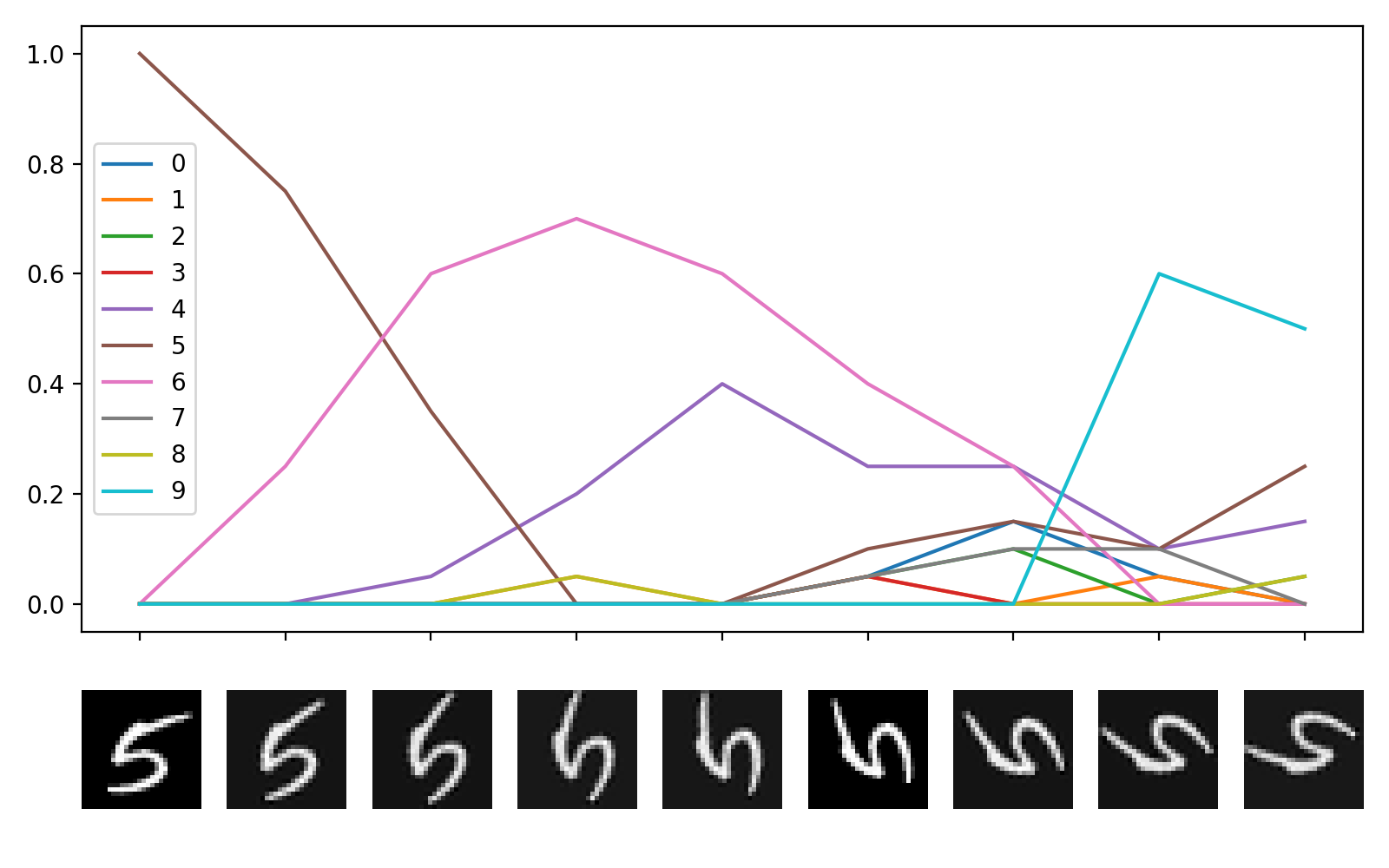}
    \hfill
    \includegraphics[width=0.48\linewidth]{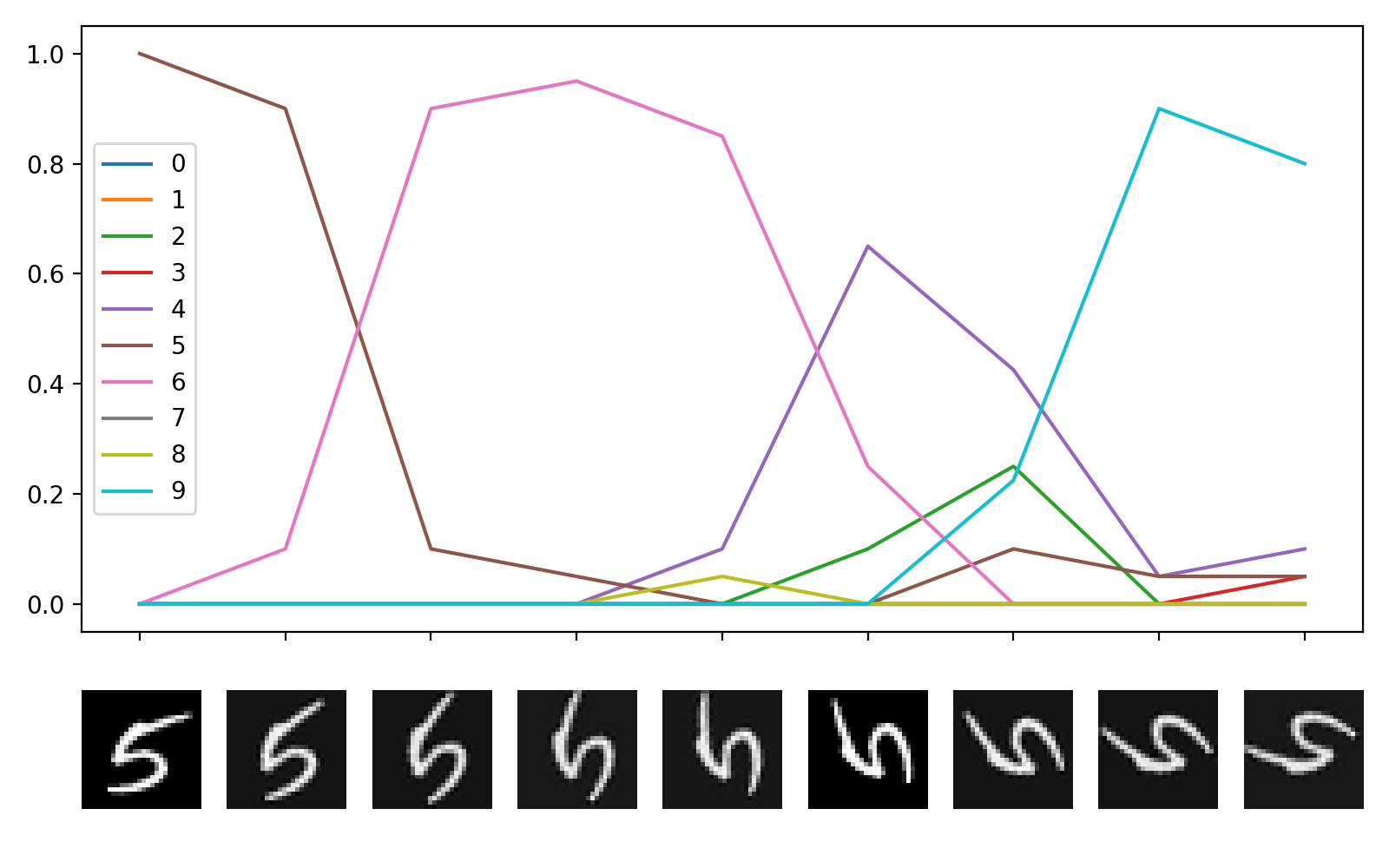}
    \includegraphics[width=0.48\linewidth]{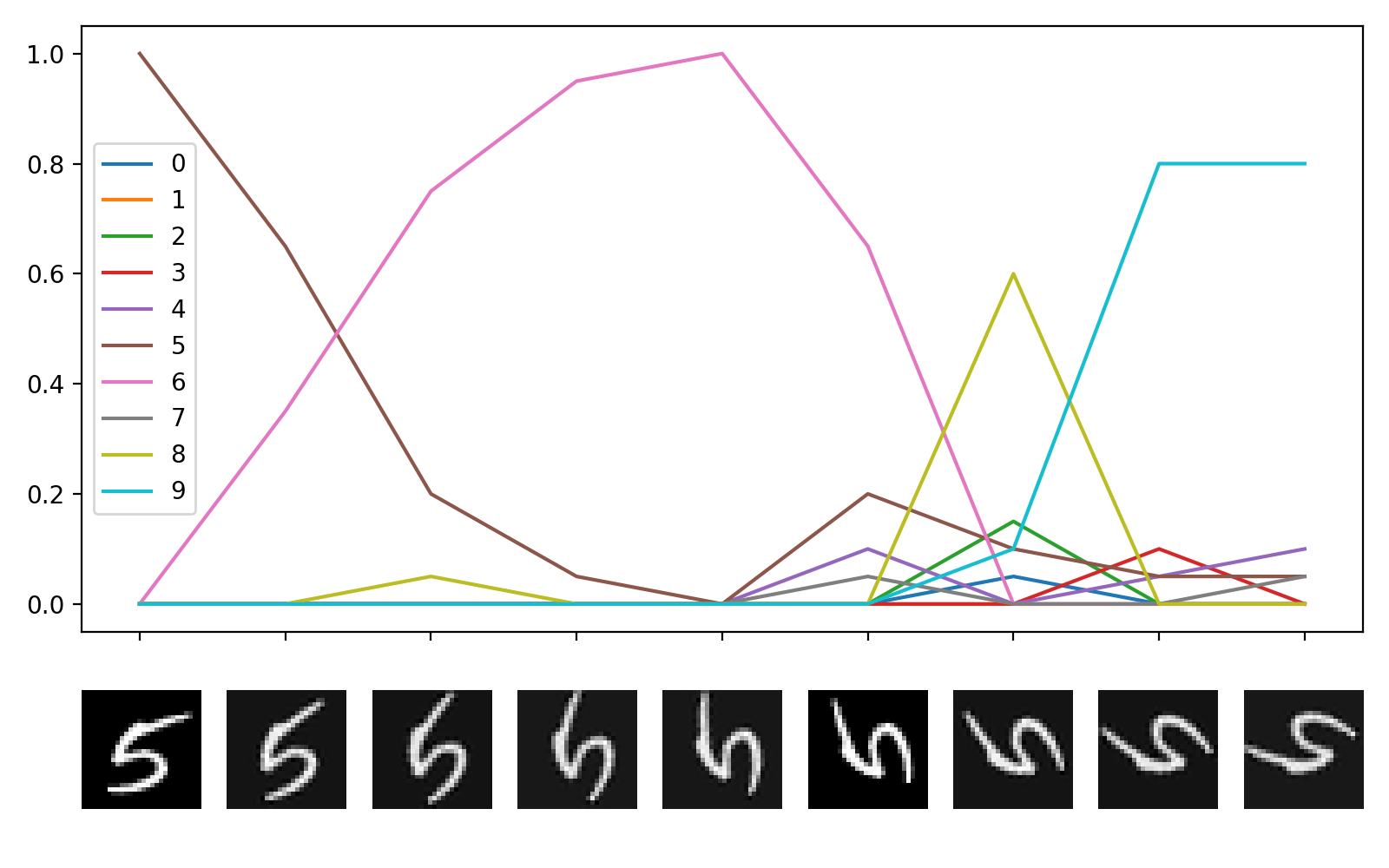}
    \hfill
    \includegraphics[width=0.48\linewidth]{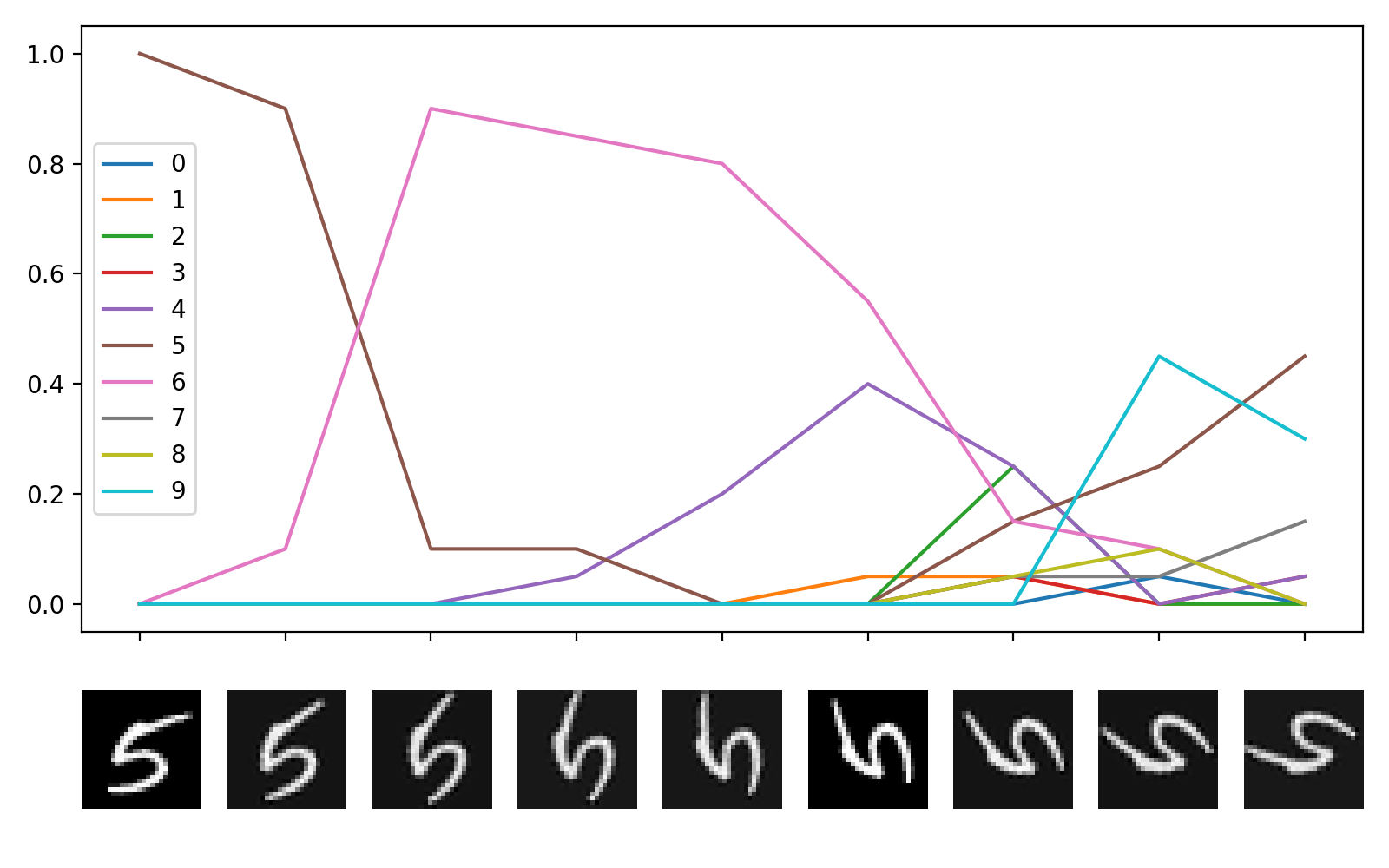}
    \caption{Posterior predictive (with 20 samples) on rotated MNIST digit 3 and 5. Top left: Diag; top right: K-Diag; bottom left: K-Linear; bottom right: K-Nonlinear.}
    \label{fig:my_label}
\end{figure}

\end{document}